\documentclass[sigconf,nosort]{acmart}
\setcopyright{none}

\usepackage{amsmath}
\usepackage{amsfonts}
\usepackage{algorithm}
\usepackage{algorithmic}
\usepackage{booktabs}
\usepackage{colortbl}
\usepackage{graphicx}
\usepackage{makecell}
\usepackage{multirow}
\usepackage{tabularx}
\usepackage[table]{xcolor}
\usepackage{xspace}

\graphicspath{{figures/}}

\definecolor{lightgray2}{gray}{0.85}

\definecolor{myred}{RGB}{237,28,80}
\definecolor{scarlet}{RGB}{255,36,0}
\definecolor{keywordred}{RGB}{200,50,60}
\definecolor{keywordgreen}{RGB}{0,150,80}
\definecolor{iceblue}{RGB}{214, 230, 245}
\definecolor{mintcream}{RGB}{240, 255, 250}
\definecolor{pastelyellow}{RGB}{254, 240, 158}
\definecolor{creamyellow}{RGB}{255,246,213}

\newcommand{\ourmethod}{LiSCP\xspace}
\newcommand{\ourmethodTT}{\textit{\textbf{\fontfamily{lmtt}\selectfont \ourmethod}}\xspace}

\newcommand{\GPTThreeFiveTurbo}{GPT-3.5-Turbo\xspace}

\newcommand{\datasetnews}{\textit{Reuter News}\xspace}
\newcommand{\datasetcode}{\textit{HumanEval Code}\xspace}
\newcommand{\datasetessay}{\textit{Student Essay}\xspace}
\newcommand{\datasetpaper}{\textit{Paper Abstract}\xspace}
\newcommand{\datasetreview}{\textit{Yelp Review}\xspace}
\newcommand{\datasetvisualnews}{\textit{VisualNews}\xspace}
\newcommand{\datasetmmimdb}{\textit{MM-IMDb}\xspace}

\newtheorem{definition}{Definition}

\newtheorem{theorem}{Theorem}

\copyrightyear{2026}
\acmYear{2026}

\acmConference[MM '26]{the 34th ACM International Conference on Multimedia}{November 10--14, 2026}{Rio de Janeiro, Brazil}
\acmDOI{}
\acmISBN{}

\renewcommand\footnotetextcopyrightpermission[1]{} % 去掉版权脚注
\begin{document}

\title[Lightweight Stylistic Consistency Profiling for LLM-Generated Text Detection]{Lightweight Stylistic Consistency Profiling: Robust Detection of LLM-Generated Textual Content for Multimedia Moderation}

\author{Siyuan Li}
\affiliation{
  \institution{School of Computer Science, Shanghai Jiao Tong University}
  \city{Shanghai}
  \country{China}
}
\email{siyuanli@sjtu.edu.cn}

\author{Aodu Wulianghai}
\affiliation{
  \institution{School of Computer Science, Shanghai Jiao Tong University}
  \city{Shanghai}
  \country{China}
}
\email{melusine.wlhad@sjtu.edu.cn}

\author{Xi Lin}
\affiliation{
  \institution{School of Computer Science, Shanghai Jiao Tong University}
  \city{Shanghai}
  \country{China}
}
\email{linxi234@sjtu.edu.cn}

\author{Xibin Yuan}
\affiliation{
  \institution{School of Computer Science, Shanghai Jiao Tong University}
  \city{Shanghai}
  \country{China}
}
\email{2022yxb@sjtu.edu.cn}

\author{Qinghua Mao}
\affiliation{
  \institution{School of Computer Science, Shanghai Jiao Tong University}
  \city{Shanghai}
  \country{China}
}
\email{mmmm2018@sjtu.edu.cn}

\author{Guangyan Li}
\affiliation{
  \institution{Institute of Automation, Chinese Academy of Sciences}
  \city{Beijing}
  \country{China}
}
\email{liguangyan2022@ia.ac.cn}

\author{Xiang Chen}
\affiliation{
  \institution{College of Computer Science and Technology, Zhejiang University}
  \city{Hangzhou}
  \country{China}
}
\email{wasdnsxchen@gmail.com}

\author{Jun Wu}
\affiliation{
  \institution{School of Computer Science, Shanghai Jiao Tong University}
  \city{Shanghai}
  \country{China}
}
\email{junwuhn@sjtu.edu.cn}

\author{Jianhua Li}
\affiliation{
  \institution{School of Computer Science, Shanghai Jiao Tong University}
  \city{Shanghai}
  \country{China}
}
\email{lijh888@sjtu.edu.cn}




\begin{abstract}
The increasing prevalence of Large Language Models (LLMs) in content creation has made distinguishing human-written textual content from LLM-generated counterparts a critical task for multimedia moderation. 
Existing detectors often rely on statistical cues or model-specific heuristics, making them vulnerable to paraphrasing and adversarial manipulations, and consequently limiting their robustness and interpretability.
In this work, we propose \ourmethodTT, a novel lightweight stylistic consistency profiling method for robust detection of LLM-generated textual content, focusing on feature stability under adversarial manipulation. 
Our approach constructs a consistency profile that combines discrete stylistic features with continuous semantic signals, leveraging stylistic stability across multimodal-guided paraphrased text variants.
Experiments spanning real-world multimedia news and movie datasets and conventional text domains demonstrate that \ourmethod achieves superior performance on in-domain detection and outperforms existing approaches by up to 11.79\% in cross-domain settings. 
Additionally, it demonstrates notable robustness under adversarial scenarios, including adversarial attacks and hybrid human-AI settings.
\end{abstract}

\maketitle
\thispagestyle{empty}
\pagestyle{empty}     

\section{Introduction}
Large Language Models (LLMs) have become a default tool for open-domain content creation, enabling high-quality generation across various domains such as news articles, product reviews, academic essays, technical documentation, and even multimedia-associated textual content (e.g., image captions, video subtitles, and multimodal platform reviews)~\cite{dubey2024llama, chen2026agentchain, su2025large, hurst2024gpt, li2024trustworthy, xu2025advancements}.
Widely deployed in both everyday tools and professional workflows, these models increasingly blur the distinction between human-written and machine-generated textual content, raising serious concerns regarding content authenticity, attribution, and accountability, which are particularly acute in multimedia content moderation scenarios where text often interacts with visual or audio modalities~\cite{sadasivan2025can, wu2025survey, yu2025evobench, li2026dsipa, arevalo2017gated, nguyen2025mtikguard}.
Consequently, the reliable detection of LLM-generated textual content has become critical for applications including academic integrity enforcement, auditing of high-stakes decisions, maintaining trust in digital communication, and ensuring the credibility of multimedia~\cite{li2025prdetect, cao2025practical}.

Despite significant advances, robust detection under realistic conditions—especially in multimedia-derived scenarios—remains an open challenge. 
Current detectors often rely on token-level statistics (e.g., likelihood- or rank-based features) or model-specific heuristics~\cite{hans2024spotting, GLTR:gehrmann2019gltr, li2026model, koike2025exagpt, abdelnabi2021adversarial, kirchenbauer2023watermark}, which tend to degrade when the generative model changes, the domain shifts, the text undergoes post-editing, or the textual content is adjusted to align with accompanying visual elements (e.g., edited image captions for misinformation propagation)~\cite{chen2025online, cheng2025beyond, zhou2025adadetectgpt, guo2023survey, Retrieval:krishna2024paraphrasing}. 
Although recent paraphrase- or re-query-based approaches reduce the need for supervision, they frequently treat paraphrasing merely as a means of score aggregation and remain tightly coupled to specific models or prompting strategies. This leads to unpredictable performance under distribution shifts, stylistic variations, hybrid human–AI compositions, or multimodal context changes (e.g., text reused across unrelated images)~\cite{lei2025pald, bao2025glimpse, zhang2024llm}. 
Moreover, most existing methods evaluate a text as a monolithic unit, offering limited insight into how stylistic patterns behave under meaning-preserving manipulations—an issue that is exacerbated when text is part of a broader multimedia ecosystem requiring cross-modal consistency~\cite{chen2025online,bao2025glimpse}.

To focus on the above problems, this work is guided by the following key research questions (RQs):
\begin{itemize}
    \item \textit{\textbf{RQ1:} How to design a detection framework that remains robust under adversarial manipulations without relying on heavyweight models?}
    \item \textit{\textbf{RQ2:} How to enhance the detection generalization across domains and multimedia-derived textual scenarios, especially when statistical features become unreliable under semantic-preserving transformations?}
\end{itemize}
In response to these two questions, we propose \ourmethodTT, a \textit{lightweight stylistic consistency profiling} method for LLM-generated textual content detection. 
Our method builds on a simple yet powerful principle: \textit{LLM authorship can be inferred from the stability of a text's stylistic patterns under meaning-preserving manipulation}. 
Instead of depending on a single text instance or aggregated detection scores, \ourmethod explicitly profiles stylistic behavior by generating multiple multimodal-aligned paraphrased variants of the input (ensuring consistency with accompanying multimodal context) and measuring consistency across them. 
Specifically, our method constructs a stylistic consistency profile that integrates: (i) \textit{discrete stylistic consistency signals} that capture surface-level invariances, and (ii) \textit{continuous semantic signals} that quantify semantic alignment across variants, including implicit alignment with the underlying multimodal context.
This profiling perspective shifts the focus from fragile, wording-specific artifacts to stability patterns that persist under paraphrasing and multimodal context adaptations, thereby aiming to address \textit{RQ2}.

To answer \textit{RQ1}, \ourmethod is designed to be both lightweight and robust in real-world detection deployment. 
Final decisions are derived from a compact consistency profile, rather than large end-to-end models or detector-specific heuristics, which improves efficiency and reduces dependence on any particular generator. 
By aggregating stability signals over meaning-preserving variants, \ourmethod emphasizes transformation-consistent signals that are difficult to remove through post-editing or rewriting.
In summary, this work makes the following contributions:
\begin{itemize}
    \item \textbf{Lightweight style profiling framework.} We propose the \ourmethodTT, a framework that profiles stylistic consistency by aggregating stability signals across multimodal-aligned paraphrased variants, robust against post-edits, model shifts.
    \item \textbf{Multi-level stylistic-semantic integration.} In this work, detection is reformulated as \textit{stylistic consistency inference under manipulation}. Our profile combines discrete stylistic and continuous semantic signals for stable patterns to enhance robustness and generalization across diverse scenarios.
    \item \textbf{Empirical validation across challenging settings.} Across diverse domains and real-world multimedia scenarios (e.g., image-text pair verification, video subtitle authentication), \ourmethod achieves state-of-the-art performance and exhibits strong robustness against adversarial attacks and hybrid human-AI compositions.
\end{itemize}

\section{Related Works}
\paragraph{Statistical and Model-Based Detectors.}
Early efforts on machine-generated text detection rely on statistical irregularities between human-written and machine-generated content, such as perplexity, entropy, or likelihood-based measures~\cite{lavergne2008detecting,hashimoto2019unifying,GLTR:gehrmann2019gltr}.
These approaches identify anomalies at the token or sequence level, forming the foundation of many modern detectors.
However, as LLMs have become increasingly fluent, such surface-level statistics are often insufficient to capture the subtle stylistic patterns exhibited by contemporary machine-generated text~\cite{jawahar2020automatic}.
Notably, in multimedia content scenarios, e.g., image-text pairs, video subtitles, and multimodal reviews, these statistical methods face additional challenges: they fail to leverage cross-modal semantic alignment cues and often degrade when text is paraphrased to adapt to multimedia context~\cite{yu2025eve, zhang2025asap, sadanandan2026psf}.
More recent work has explored supervised training-based detectors, typically fine-tuning large models to distinguish human-written and machine-generated textual content~\cite{yang2024survey}.
Representative systems such as GPTZero and OpenAI’s classifier train RoBERTa-style models on labeled corpora to learn discriminative representations~\cite{GPTZero,OpenAI-CLS2019}. 
While effective in controlled settings, these detectors frequently suffer from domain shift, limited cross-model generalization, and sensitivity to post-editing or rewriting~\cite{hu2024radar,RAIDAR:mao2024detecting,dai2026hlpd}. 
In contrast, our work avoids reliance on heavyweight classifiers or task-specific supervision, instead focusing on stylistic stability and multimodal-guided semantic consistency.

\paragraph{Paraphrase and Perturbation-Based Detection.} 
To mitigate overfitting to a single text instance, several methods leverage perturbations or paraphrasing to improve robustness. 
DetectGPT~\cite{mitchell2023detectgpt} exploits likelihood curvature by comparing model scores before and after perturbations, based on the observation that LLM-generated text often lies near local likelihood maxima and becomes unstable under controlled rewrites. 
Fast-DetectGPT~\cite{bao2024fast} improves efficiency through a more lightweight perturbation routine. %, while retaining similar performance.
Other approaches, such as Binoculars~\cite{hans2024spotting} and BiScope~\cite{guo2024biscope}, contrast scores across different language models to enhance generalization beyond a single generator.
Despite their effectiveness, most perturbation-based methods treat paraphrasing as a mechanism for score aggregation rather than a signal in its own right~\cite{shportko2025paraphrasing,mao2025learning,li2025styledecipher,li2025prdetect}. 
As a result, they continue to rely on global likelihood statistics and offer limited insight into fine-grained stylistic properties, especially in multimedia scenarios where text stylistic patterns are often constrained by paired visual content~\cite{cao2025practical}. 
Our work differs fundamentally by explicitly modeling stylistic consistency across multimodal-guided paraphrased variants, shifting the focus from scores to stability patterns that persist under adversarial manipulation and multimedia context adaptation.

\paragraph{Robust Detection in Real-World Settings.} 
Recent studies have increasingly emphasized real-world constraints such as hybrid human-AI authorship, partial access to proprietary models, and streaming text scenarios. 
PALD~\cite{lei2025pald} estimates the proportion of machine-generated content at the sentence level, enabling partial authorship analysis in mixed documents. 
GLIMPSE~\cite{bao2025glimpse} bridges white-box and black-box settings by reconstructing probability distributions from limited observations, improving robustness across proprietary models. 
Additional work explores non-parametric distribution comparison~\cite{R-Detect:song2025deep} or sequential hypothesis testing for online detection~\cite{chen2025online}. 
Existing robust detection methods improve applicability, but they often output a single global score, remain sensitive to paraphrasing or light editing, and overlook cross-modal semantic constraints.
In contrast, our method explicitly profiles stylistic behavior under meaning-preserving manipulations and multimodal alignment, enabling robust detection across domains and adversarial settings without model-specific assumptions.

\section{Lightweight Stylistic Consistency Profiling for LLM-Generated Content Detection}
In this section, we formally develop our lightweight framework for detecting LLM-generated textual content,
a critical task in multimedia content moderation (e.g., verifying text authenticity in image-text reviews, video subtitles, and multimodal academic papers).
We first define the paraphrase-induced text space and stylistic stability,
then introduce a multi-level consistency profile derived from discrete and continuous feature mappings.
Based on this formulation, we present the detection rule and efficient algorithmic realization,
which can be seamlessly integrated into real-world multimedia moderation pipelines.

\subsection{Detection Problem Formulation and Paraphrase Space}
Let $\mathcal{X}$ denote the space of all texts (the core object of detection in multimedia content).
Given an input text $x \in \mathcal{X}$ (often paired with visual/audio content in multimedia scenarios),
we aim to predict its authorship label $y \in \{0,1\}$ (human vs. LLM-generated).

We leverage transformation stability to distinguish authorship.
Under a predefined prompt set \(\mathcal{P}\), a multimodal-guided paraphrasing operator \(M_I\) maps the input pair \((I,x)\) to a set of meaning-preserving rewrites:
\begin{equation}
    \mathcal{P}(x;I)=\{\hat{x}_k \mid \hat{x}_k = M_I(I,x,p_k),\; p_k\in\mathcal{P},\; k=1,\dots,K\}.
\end{equation}
We enforce semantic preservation by filtering out variants with semantic similarity to $x$ below a threshold $\delta$, ensuring rewrites remain consistent with the original text's core meaning, which is an essential property for text detection in multimedia content (e.g., avoiding off-topic rewrites in image captions).
Given the constructed paraphrase set $\{x\} \cup \mathcal{P}(x;I)$, the detection objective is to learn a stylistic consistency profile that captures invariant patterns across rewrites, enabling robust discrimination even in multimedia scenarios with noisy or manipulated text.

\subsection{Multi-Level Stylistic Consistency Profiling}
We formulate detection as an inference problem over transformation stability patterns.
Instead of analyzing a single text instance, we construct a structured profile that captures how stylistic signals behave across paraphrased variants.

\begin{definition}[Stylistic Consistency Profile]
    A stylistic consistency profile is an aggregated vector representation
    \begin{equation}
        v: \mathcal{X}\to\mathbb{R}^d,
    \end{equation}
    where \(v(x)\) is constructed from the stability measurements between \(x\) and its paraphrased variants \(\hat{x}\in \mathcal{P}(x; I)\).
\end{definition}
The profile $v(x)$ is constructed by integrating surface-level discrete consistency signals and semantic-level continuous consistency signals, as detailed below.

\begin{figure}[!t]
    \centering
    \includegraphics[width=\linewidth]{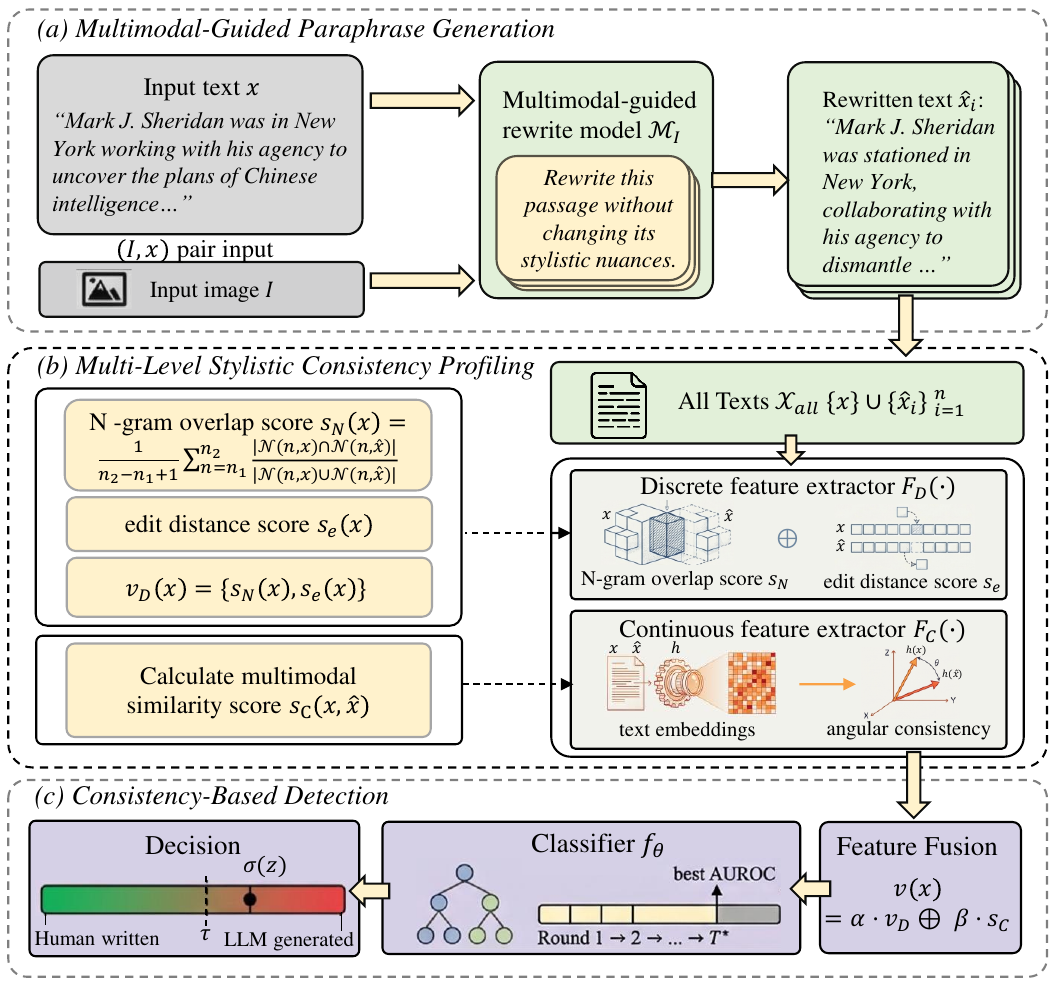}
    \caption{Overview of the \ourmethod. 
    \textbf{(a) Multimodal-Guided Paraphrase Generation:} Given an input pair \((I,x)\), a rewrite model $\mathcal{M}_I$ generates semantically consistent variants \(\{\hat{x}_i\}\). 
    \textbf{(b) Multi-Level Stylistic Consistency Profiling:} Discrete features \(s_D(x,\hat{x})\) and continuous features \(s_C(x,\hat{x})\) are extracted to capture linguistic stability. 
    \textbf{(c) Consistency-Based Detection:} Features are aggregated into \(v(x)\) and fed into a gradient-boosted classifier to predict whether \(x\) is LLM-generated.}
    \label{figure:pipeline}
\end{figure}

\paragraph{\textit{Surface-Level Discrete Consistency Profiling.}}
For the tokenized text $x$, its n-gram set is defined as $\mathcal{N}(n,x) = \{(w_i,\dots,w_{i+n-1})\}_{i=1}^{|x|-n+1}$.
The normalized n-gram stability across a range $[n_1,n_2]$ is:
\begin{equation}
\label{eq:ngram_stability}
    s_N(x,\hat{x}) =
    \frac{1}{n_2-n_1+1}
    \sum_{n=n_1}^{n_2}
    \frac{|\mathcal{N}(n,x) \cap \mathcal{N}(n,\hat{x})|}
         {|\mathcal{N}(n,x) \cup \mathcal{N}(n,\hat{x})|}.
\end{equation}
To further capture fine-grained lexical perturbations, we incorporate edit-based consistency.
Let $\mathcal{D}(x,\hat{x})$ denote the Levenshtein distance defined via dynamic programming.
We define the normalized edit stability as:
\begin{equation}
\label{eq:edit_stability_general}
    s_E(x,\hat{x}) =
    1 - \frac{\mathcal{D}(x,\hat{x})}{\max(|x|,|\hat{x}|)}.
\end{equation}
Furthermore, the discrete consistency vector between $x$ and $\hat{x}$ is $\mathbf{s}_D(x,\hat{x}) = \bigl[s_N(x,\hat{x}),\; s_E(x,\hat{x})\bigr]^\top$, where human content typically exhibits lower stability (flexible rewrites) and LLM-generated content exhibits higher stability (structural invariance)—a pattern that holds even in multimedia-derived texts.

\paragraph{\textit{Semantic-Level Continuous Consistency Profiling.}}
We capture continuous consistency using a shared text encoder \(\xi\). 
Given \(x\) and \(\hat{x}\), their contextual embeddings are \(h(x)=\mathrm{Pool}(\xi(x))\) and \(h(\hat{x})=\mathrm{Pool}(\xi(\hat{x}))\). 
We then compute a normalized angular consistency score
\begin{equation}
\label{eq:semantic_consistency}
    s_C(x,\hat{x}) = 1 - \frac{1}{\pi}\arccos\left(
    \frac{h(x)^\top h(\hat{x})}{\|h(x)\|_2 \|h(\hat{x})\|_2}
    \right).
\end{equation}
Although \(s_C(x,\hat{x})\) is defined over textual representations, both paraphrases \(\hat{x}\) and the resulting consistency profile are obtained under the paired image context \(I\) through the multimodal-guided rewrite process.
This measure captures semantic stability beyond surface-level wording in multimedia scenarios.

% \subsubsection{Profile Aggregation across Paraphrases}
Given the demands of real-time multimedia content moderation, this profile is engineered to be lightweight and compact, ensuring efficient inference. 
For the text $x$ with its paraphrase set $\mathcal{P}(x; I)$, the final consistency profile is derived by aggregating pairwise stylistic and semantic stability features:
\begin{equation}
\label{eq:profile_aggregation}
    v(x) =
    \frac{1}{|\mathcal{P}(x; I)|}
    \sum_{\hat{x} \in \mathcal{P}(x; I)}
    \Bigl[
    \alpha \cdot \mathbf{s}_D(x,\hat{x})
    \;\oplus\;
    \beta \cdot s_C(x, \hat{x})
    \Bigr],    
\end{equation}
where $\alpha,\beta > 0$ are scaling coefficients and $\oplus$ denotes the feature vector concatenation.

\begin{algorithm}[!tb]
    \caption{Multimodal-Guided Stylistic Consistency Detection}
    \label{alg:scp}
    % \small
    \begin{algorithmic}[1]
        \STATE \textbf{Input}: multimodal input pair \((I,x)\), multimodal paraphraser \(M_I\), prompt set \(\mathcal{P}\), encoder \(\xi\), classifier \(f_\theta\) \\
        \STATE \textbf{Parameters}: number of paraphrases \(K\), fusion weights \(\alpha,\beta\), decision threshold \(\tau\) \\
        \STATE \textbf{Output}: Predicted label $\hat{y}$           
        \STATE $\mathcal{X}_p \leftarrow \varnothing, \enspace S_N, S_E, S_C \leftarrow \varnothing$ 
        \FOR{$k = 1,\dots,K$}
            \STATE $\hat{x}_k \leftarrow M_I(I, x, p_k)$
            \STATE $\mathcal{X}_p \leftarrow \mathcal{X}_p \cup \{\hat{x}_k\}$
        \ENDFOR
        % \hfill \textcolor{keywordgreen}{ $\triangleright$ \textit{// Step 1: Meaning-preserving paraphrasing} }
        \FOR{each $\hat{x} \in \mathcal{X}_p$}
            \STATE $s_N \leftarrow \textsc{NgramStability}(x,\hat{x})$
            \STATE $S_N \leftarrow S_N \cup \{s_N\}$
            % \hfill \textcolor{keywordgreen}{ $\triangleright$ \textit{// Surface-level lexical stability} }
            \STATE $s_E \leftarrow \textsc{EditStability}(x,\hat{x})$
            \STATE $S_E \leftarrow S_E \cup \{s_E\}$
            % \hfill \textcolor{keywordgreen}{ $\triangleright$ \textit{// Surface-level edit stability} }
            \STATE $\mathbf{h}_x \leftarrow \xi(x), \enspace \mathbf{h}_{\hat{x}} \leftarrow \xi(\hat{x})$
            \STATE $s_C \leftarrow \textsc{SemanticConsistency}(\mathbf{h}_x, \mathbf{h}_{\hat{x}})$
            \STATE $S_C \leftarrow S_C \cup \{s_C\}$
            % \hfill \textcolor{keywordgreen}{ $\triangleright$ \textit{// Representation-level semantic stability} }
        \ENDFOR
        
        \STATE {\small $\bar{s}_N \leftarrow \frac{1}{|S_N|} \sum\limits_{s \in S_N}s$, $\bar{s}_E \leftarrow \frac{1}{|S_E|}\sum\limits_{s \in S_E}s$, $\bar{s}_C \leftarrow \frac{1}{|S_C|}\sum\limits_{s \in S_C}s$}
        \STATE $v_D(x) \leftarrow (\bar{s}_N, \bar{s}_E)$
        \STATE $v(x) \leftarrow \alpha \cdot v_D(x) \oplus \beta \cdot \bar{s}_C$
        % \hfill \textcolor{keywordgreen}{ $\triangleright$ \textit{// Construct consistency profile} }
        \STATE $z \leftarrow f_\theta(v(x))$, \enspace $\hat{y} \leftarrow \mathbb{I}[\sigma(z) \ge \tau] $
        \STATE \textbf{return} $\hat{y}$
    \end{algorithmic}
\end{algorithm}

\subsection{Consistency-based Detection Rule}
Let $v(x) \in \mathbb{R}^d$ denote the aggregated consistency profile.
We model the conditional likelihood of authorship via a parametric decision function $f_\theta : \mathbb{R}^d \rightarrow \mathbb{R}$.
The detection score is defined as: $z(x) = f_\theta(v(x))$, and the predicted label is obtained via thresholding: $\hat{y}=\mathbb{I}\!\left[\sigma(z(x))\ge\tau\right]$, where $\sigma(\cdot)$ is the sigmoid function and $\tau$ is a fixed decision threshold.
The theoretical guarantee for the separability is as follows:

\begin{theorem}[Expected Separation of Stability Features]
\label{theorem}
    To analyze the separability of human-written and LLM-generated textual content, let \(v(x) \in \mathbb{R}^d\) denote the stylistic consistency profile induced by an input text \(x\) and its multimodal-guided paraphrased variants. 
    Assume that there exists a constant \(\epsilon>0\) such that the class-conditional expectations satisfy
    \begin{equation}
        \mathbb{E}[v(x)\mid y=1] - \mathbb{E}[v(x)\mid y=0]\succeq \epsilon\,\mathbf{1},
    \end{equation}
    where \(\mathbf{1}\in\mathbb{R}^d\) denotes the all-ones vector, \(y\in\{0,1\}\) is the binary authorship label, and \(\succeq\) denotes element-wise inequality. 
    Then there exists a linear scoring function \(f_\theta(v)=\theta^\top v\) that achieves positive expected margin separation between the two classes.
\end{theorem}

\begin{proof}
    Let \(   \mu_1=\mathbb{E}[v(x)\mid y=1], \qquad \mu_0=\mathbb{E}[v(x)\mid y=0]  \) denote the class-conditional mean consistency profiles. 
    By assumption, we have \(  \mu_1-\mu_0 \succeq \epsilon\,\mathbf{1}. \)
    Consider the linear scoring function defined by \( \theta=\mathbf{1}\in\mathbb{R}^d. \)
    Then
    \begin{equation}
        \theta^\top\mu_1-\theta^\top\mu_0 = \sum_{k=1}^{d}(\mu_{1,k}-\mu_{0,k}) \ge d\epsilon >0.
    \end{equation}
    Consequently, there exists a constant \(\Delta=d\epsilon>0\) such that
    \begin{equation}
        \mathbb{E}[\theta^\top v(x)\mid y=1] - \mathbb{E}[\theta^\top v(x)\mid y=0]  \ge \Delta.
    \end{equation}
    Therefore, \(f_\theta(v)=\theta^\top v\) achieves a strictly positive expected separation margin between the two classes.
\end{proof}
The theorem confirms our stylistic consistency profile’s discriminative capability for computationally constrained multimedia content moderation pipelines without heavy neural networks.

Algorithm~\ref{alg:scp} explicitly decouples text paraphrase set generation, stylistic stability measurement, and final decision making into three modular stages.
In the first stage, the paraphrasing component explores a local semantic neighborhood of the multimodal input pair \((I,x)\) by generating a paraphrase set \(\mathcal{P}(x;I)=\{\hat{x}_k\}_{k=1}^{K}\), which provides multiple meaning-preserving variants for subsequent analysis.
In the second stage, the consistency extraction module maps each original-paraphrase pair \((x,\hat{x}_k)\) into a set of stability signals \(\mathbf{s}(x,\hat{x}_k)=\bigl[\mathbf{s}_D(x,\hat{x}_k),\, s_C(x,\hat{x}_k)\bigr]\), capturing both surface-level stylistic invariance and continuous semantic consistency.
In the final stage, all pairwise stability signals are aggregated into a fixed-dimensional consistency profile $ v(x)=\frac{1}{K}\sum_{k=1}^{K}\mathbf{s}(x,\hat{x}_k)$, which is independent of the paraphrase count \(K\) and thus enables efficient downstream inference without increasing model capacity.

\section{Experiments}
We evaluate the proposed \ourmethod from five complementary perspectives: in-domain detection performance, cross-domain generalization, robustness under adversarial and hybrid settings, interpretability of the learned stability profile, and sensitivity to the choice of semantic encoder.

\subsection{Experimental Setup}
\paragraph{Datasets.}
We evaluate \ourmethod on datasets selected from two complementary perspectives: widely-adopted benchmarks in the MGT detection literature for comparability with prior work, and datasets sourced from multimedia platforms to evaluate applicability in real-world multimedia content moderation scenarios. The former includes five conventional text domains and the large-scale \textit{RAID} benchmark; the latter includes \datasetvisualnews and \datasetmmimdb, where textual content is inherently paired with visual media.
% We evaluate \ourmethod on six datasets spanning diverse domains: Reuter news articles, student essays, HumanEval code, Yelp reviews, academic papers, and the large-scale \textit{RAID} benchmark.

\textbf{News Domain (\datasetnews).}
Using \textit{ChatGPT} (\textit{davinci})~\cite{verma2024ghostbuster}, we generate LLM-written news articles paired with human-written news from the \texttt{Reuter\_50\_50} dataset~\cite{houvardas2006n}.

\textbf{Essay Domain (\datasetessay).}
Human-written essays are sourced from IvyPanda~\cite{ivypanda}, a repository of student-written essays, while LLM-generated essays are produced using \textit{ChatGPT}~\cite{verma2024ghostbuster}.

\textbf{Code Domain (\datasetcode).}
We use the \datasetcode dataset~\cite{RAIDAR:mao2024detecting} for human-written code and use \textit{\GPTThreeFiveTurbo} to generate machine-written code.
This domain is included to test whether our method remains effective on highly structured content.

\textbf{Review Domain (\datasetreview).}
Using \textit{\GPTThreeFiveTurbo}~\cite{RAIDAR:mao2024detecting}, we generate LLM-written reviews and pair them with human-written reviews from Yelp~\cite{zhang2015character}.
% This domain is included to assess performance on informal, user-generated text with diverse expressions and colloquial phrasing.

\textbf{Paper Abstract Domain (\datasetpaper).}
We sample 500 human-written abstracts from ACL 2023, 2024 papers and use \textit{\GPTThreeFiveTurbo} to generate LLM-written paper abstracts.

\textbf{Visual News Domain (\datasetvisualnews).}
Human-written articles are from \textit{VisualNews}~\cite{liu2021visualnews}, collected from four major multimedia news outlets, with each article paired with a news image. LLM-generated counterparts are produced using \textit{\GPTThreeFiveTurbo}.

\textbf{Movie Description Domain (\datasetmmimdb).}
Human-written plot descriptions are from \textit{MM-IMDb}~\cite{arevalo2017gated}, a multimodal benchmark pairing movie posters with editorial synopses, while LLM-generated descriptions are produced using \textit{\GPTThreeFiveTurbo}.

\textbf{\textit{RAID} Benchmark.}
The official \textit{RAID} benchmark~\cite{dugan2024raid} includes over 10 million documents from 11 LLMs, testing generalization across generators, decoding strategies, and attack conditions.

\paragraph{Baselines.} We compare the proposed \ourmethod against several representative baseline detectors from multiple categories:

\textbf{GPTZero~\cite{GPTZero}.}
GPTZero is a commercial classifier that relies on handcrafted features and shallow syntactic heuristics. We use its official API for implementation.

\textbf{DetectGPT~\cite{mitchell2023detectgpt}.}
DetectGPT identifies LLM-generated content by examining changes in the curvature of log-probability under small input perturbations.

\textbf{Ghostbuster~\cite{verma2024ghostbuster}.}
Ghostbuster is a black-box detector that enforces cross-domain generalization by ensembling features from multiple weaker models. 

\textbf{RAIDAR~\cite{RAIDAR:mao2024detecting}.}
RAIDAR detects machine-generated texts by rewriting the input and comparing the resulting differences to identify discrepancies between human and LLMs.

\textbf{Fast-DetectGPT~\cite{bao2024fast}.}
Fast-DetectGPT is a more efficient zero-shot detector than DetectGPT, which approximates probability curvature signals via conditional sampling.

\textbf{R-Detect~\cite{R-Detect:song2025deep}.}
R-Detect applies a nonparametric kernel relative test to determine whether a test text is statistically closer to a human or a machine distribution.

\paragraph{Implementation Details.}
We use AUROC as the primary metric for ranking quality, and we also report the best F1 score obtained by sweeping the decision threshold. % on the validation set and applying it to the test set. 
For \textit{Fast-DetectGPT}, \textit{Binoculars}, \textit{R-Detect}, and \textit{DetectGPT}, we use their official implementations but re-evaluate them under a common protocol: AUROC is computed from raw scores, and the F1 score is obtained via threshold sweeping on the same split as our method.
For \textit{Ghostbuster}, the original work relies on \textit{GPT-Ada} and \textit{GPT-Davinci}, which are now deprecated. We replace them with \textit{\GPTThreeFiveTurbo} as drop-in substitutes, keeping all other hyperparameters unchanged.  
For \textit{DetectGPT}, we follow the original paper and use \textit{T5-3B} as the perturbation model.  
All LLM calls are made in a batchified manner to control variance across methods.
For \ourmethod, we use \textit{\GPTThreeFiveTurbo} as the paraphrase model $\mathcal{M}$ and SBERT as the default encoder $\xi$. % unless otherwise specified.  
The classifier $f$ is instantiated as a gradient-boosted tree with early stopping based on validation AUROC.  
To ensure fairness, F1 scores reported for all baselines in subsequent tables and figures are computed by threshold sweeping on the same held-out validation splits, and RAID configurations follow~\cite{R-Detect:song2025deep} unless otherwise noted.

% ----------------------------
\begin{table}[!t]
  \centering
  \scriptsize
  \setlength{\tabcolsep}{2pt} 
  \caption{Main detection AUROC of the LLM-generated content across \datasetnews, \datasetcode, \datasetessay,  \datasetreview, \datasetvisualnews, and \datasetmmimdb datasets. \textbf{Bold} indicates the best performance.}
  \resizebox{\linewidth}{!}{
  \begin{tabular}{lcccccccc}
    \toprule
    \textbf{Method} & \textbf{News} & \textbf{Code} & \textbf{Essay} & \textbf{Yelp} & \textbf{VisualNews} & \textbf{MM-IMDb} \\ % & \textbf{Avg} \\
    \midrule
    Entropy & 0.4246 & 0.4306 & 0.4808 & 0.4697& 0.4746 & 0.4358 \\ % & 0.4527 \\
    \rowcolor{gray!10} Rank & 0.6560 & 0.5348 & 0.6849 & 0.6819 & 0.5412 & 0.5292 \\ % & 0.6047 \\
    LogRank & 0.7438 & 0.5350 & 0.6758 & 0.5294 & 0.6712 & 0.5068 \\ % & 0.6102 \\
    \rowcolor{gray!10} RoBERTa-base & 0.7024 & 0.4217 & 0.6317 & 0.4723 & 0.5358 & 0.4076 \\ % & 0.5286 \\
    RoBERTa-large & 0.7301 & 0.4692 & 0.3325 & 0.4061 & 0.5674 & 0.4371 \\ % & 0.4904 \\
    \rowcolor{gray!10} DetectGPT & 0.8213 & 0.5267 & 0.6410 & 0.6342 & 0.8124 & 0.6058 \\ % & 0.6735 \\
    Ghostbuster & 0.6401 & 0.5378 & 0.5798 & 0.6691 & 0.7122 & 0.5684 \\ % & 0.6179 \\
    \rowcolor{gray!10} Fast-DetectGPT & 0.9486 & 0.6679 & 0.9206 & 0.6230 & 0.9173 & 0.6446 \\ % & 0.7885 \\
    RAIDAR & 0.8956 & 0.8173 & 0.9091 & 0.8616 & 0.9312 & 0.8246 \\ % & 0.8732 \\
    \rowcolor{gray!10} R-Detect & \textbf{0.9817} & 0.6490 & 0.7629 & 0.7121 & 0.9650 & 0.7048 \\ % & 0.7959 \\
    \midrule
    \rowcolor{creamyellow!50}
    \textbf{\ourmethod (Ours)} & 0.9356 & \textbf{0.8108} &  \textbf{0.9455} & \textbf{0.8718} & \textbf{0.9746} & \textbf{0.9576} \\ % & \textbf{0.9172} \\
    \bottomrule
    \end{tabular}
    }
    \label{table:Main-detection}
\end{table}
% ------------------------

\subsection{Detection Performance}

\paragraph{In-Domain Detection Performance.}
We first examine the core detection performance of LiSCP on both conventional text domains and multimedia-associated datasets.
\autoref{table:Main-detection} reports AUROC scores across six representative domains, spanning conventional text domains and multimedia content scenarios. 
\ourmethod achieves the best average performance and either outperforms or closely matches the strongest baseline in each individual domain.
Notably, \ourmethod demonstrates strong performance in domains with structurally diverse content such as \datasetessay and \datasetcode, significantly outperforming likelihood-based detectors and supervised classifiers that rely on surface-level statistics.
Furthermore, \ourmethod achieves particularly strong results on multimedia content domains, \datasetvisualnews and \datasetmmimdb, outperforming all baselines by a clear margin. This confirms the domain-agnostic nature of stylistic consistency as a detection signal, extending naturally to multimedia content moderation without additional adaptation.

\paragraph{Evaluation on the RAID Benchmark.}
\begin{table}[!t]
  \centering
  \small
  % \footnotesize
  \setlength{\tabcolsep}{5pt}
  \caption{Main detection AUROC on the RAID benchmark under six mixed data and adversarial attack configurations.}
  \label{table:RAID}
  \resizebox{\linewidth}{!}{
  \begin{tabular}{lcccccc}
    \toprule
    \textbf{Method} & \textbf{Mix1} & \textbf{Mix2} & \textbf{Mix3} & \textbf{Att1} & \textbf{Att2} & \textbf{Att3}\\ % & Writing
    \midrule
    DetectGPT & 0.6437 & 0.6632 & 0.4987 & 0.5931 & 0.5111 & 0.4554 \\
    \rowcolor{gray!10}
    Ghostbuster & 0.7013 & 0.6643 & 0.5388 & 0.6645 &  0.6465 & 0.6356 \\
    Fast-DetectGPT & 0.7596 & 0.7901 & 0.7620 & 0.7324 &  0.8410 & 0.7129 \\
    \rowcolor{gray!10}
    RAIDAR & 0.8090 & 0.6875 & 0.6500 & 0.7876 & 0.6476 & 0.7112 \\
    R-Detect & 
    0.8643 & 0.7656 & 0.7650 & 0.7855 & \textbf{0.7829} & 0.7163 \\
    \midrule
    \rowcolor{creamyellow!50} 
    \textbf{\ourmethod (Ours)} & 
    \textbf{0.8957} & \textbf{0.7958} & \textbf{0.7813} & \textbf{0.8268} & 0.7714 & \textbf{0.7608} \\
    \bottomrule
    \end{tabular}
   }
\end{table}
We further evaluate our method on the RAID benchmark, which consists of multiple generators, genres, decoding strategies, and adversarial attacks. 
Following prior work~\cite{dugan2024raid}, we test the model's ability to generalize across unseen generator-attack configurations.
As shown in~\autoref{table:RAID}, \ourmethod performs competitively with the strongest baselines in clean settings and often surpasses them when evaluated on attacked configurations.
In particular, detectors that rely heavily on raw likelihoods experience a sharp performance drop under paraphrasing and corruption. In contrast, our stability-based approach remains effective and provides informative signals even in the presence of such adversarial perturbations.

\subsection{Cross-Domain Generalization}
\begin{table*}[!ht]
    \centering
    \footnotesize
    \setlength{\tabcolsep}{5.5pt}
    \caption{Cross-domain generalization F1 score on ID-OOD splits. 
    Each row group represents a source domain (used for training), while each column shows the target domain used for evaluation. 
    \textbf{OOD-Avg} denotes the average F1 score on out-of-domain targets.
    \textbf{Bold} indicates the best performance, and \underline{underline} indicates the second best. }
    \label{figure:OOD}
    \resizebox{\linewidth}{!}{
    \begin{tabular}{l|*{2}{|*{5}{c}}}
    \toprule
    % \rowcolor{gray!30}
    % \multirow{2}{*}{\textbf{Method}} 
    & \multicolumn{5}{c}{\textbf{\texttt{Paper Abstract}}} & \multicolumn{5}{c}{\textbf{\texttt{HumanEval Code}}} \\
    % \cmidrule(lr){2-6} \cmidrule(lr){7-11} %\cmidrule(lr){12-16}
    \textbf{Method} & \datasetpaper & \textit{HumanEval} & \datasetessay & \datasetnews & \datasetreview & \datasetpaper & \textit{HumanEval} & \datasetessay & \datasetnews & \datasetreview \\ \midrule 
    Entropy & 
    56.75$\pm$0.984 & 28.14$\pm$0.719 & 36.52$\pm$0.507 & 48.14$\pm$0.813 & 41.83$\pm$1.231 & 
    26.10$\pm$0.753 & 58.23$\pm$1.093 & 35.68$\pm$0.724 & 36.17$\pm$0.148 & 40.22$\pm$1.062 \\ 
    \rowcolor{gray!10} Rank & 
    54.60$\pm$0.218 & 30.97$\pm$0.413 & 35.62$\pm$0.706 & 38.79$\pm$1.038 & 52.63$\pm$1.071 & 
    35.11$\pm$1.202 & 48.51$\pm$0.308 & 40.10$\pm$0.650 & 28.31$\pm$0.828 & 42.23$\pm$0.391 \\ 
    LogRank & 
    64.05$\pm$0.554 & 29.10$\pm$0.632 & 43.71$\pm$0.752 & 42.35$\pm$1.036 & 51.72$\pm$0.953 & 
    42.89$\pm$0.865 & 52.09$\pm$0.733 & 33.67$\pm$1.032 & 39.16$\pm$1.026 & 35.31$\pm$0.973 \\ 
    \rowcolor{gray!10} RoBERTa-base & 
    57.78$\pm$0.867 & 38.72$\pm$0.844 & 35.91$\pm$0.592 & 48.02$\pm$0.508 & 56.78$\pm$1.215 & 
    39.01$\pm$0.295 & 46.74$\pm$1.059 & 35.10$\pm$0.520 & 39.80$\pm$0.769 & 47.90$\pm$0.695 \\ 
    RoBERTa-large & 
    63.40$\pm$0.850 & \underline{41.48$\pm$0.921} & 50.81$\pm$0.905 & 55.27$\pm$0.596 & 61.82$\pm$1.010 & 
    39.15$\pm$0.935 & 58.10$\pm$0.995 & 47.92$\pm$1.003 & 39.58$\pm$1.256 & 52.09$\pm$0.958 \\ 
    \rowcolor{gray!10} DetectGPT & 
    83.33$\pm$0.983 & 33.64$\pm$1.271 & 53.96$\pm$0.592 & 64.14$\pm$0.915 & 59.24$\pm$1.037 & 
    41.39$\pm$0.854 & 56.23$\pm$0.470 & 69.55$\pm$0.792 & 46.95$\pm$0.849 & 47.05$\pm$0.384 \\
    Ghostbuster & 
    74.52$\pm$0.850 & 31.20$\pm$0.519 & 40.31$\pm$1.248 & 62.75$\pm$0.497 & \underline{64.81$\pm$0.588} & 39.65$\pm$0.550 & 61.27$\pm$0.681 & \textbf{71.90$\pm$0.679} & 46.71$\pm$1.260 & 51.38$\pm$1.054 \\ 
    Fast-DetectGPT & 
    \underline{86.50$\pm$1.320} & 36.48$\pm$0.942 & \underline{56.07$\pm$1.419} & 65.12$\pm$0.920 & 63.80$\pm$1.161 & 45.05$\pm$1.007 & 62.48$\pm$1.256 & 70.37$\pm$1.310 & 58.45$\pm$0.852 & 63.40$\pm$0.956 \\
    \rowcolor{gray!10} RAIDAR & 
    78.34$\pm$0.526 & 37.05$\pm$0.543 & 53.86$\pm$0.652 & \textbf{73.62$\pm$1.059} & 60.85$\pm$0.478 & 
    \underline{46.28$\pm$0.628} & 75.01$\pm$0.343 & 47.24$\pm$0.938 & \underline{59.42$\pm$0.609} & \textbf{79.51$\pm$ 1.203} \\
    R-Detect & 
    85.94$\pm$0.950 & 40.53$\pm$1.034 & 54.13$\pm$1.192 & 64.05$\pm$0.804 & 62.16$\pm$0.621 & 
    42.15$\pm$0.452 & \underline{76.50$\pm$0.925} & 69.72$\pm$0.806 & 56.10$\pm$0.763 & 64.38$\pm$0.629 \\
    \midrule
    \rowcolor{creamyellow!50}
    \textbf{\ourmethod (Ours)} & \textbf{91.54$\pm$0.340} & \textbf{45.14$\pm$0.462} & \textbf{67.16$\pm$0.215} & \underline{66.56$\pm$0.409} & \textbf{70.09$\pm$0.317} & \textbf{46.43$\pm$0.683} & \textbf{83.33$\pm$0.429} & \underline{70.79$\pm$0.155} & \textbf{62.70$\pm$0.517} & \underline{67.48$\pm$0.438} \\
    \midrule % \midrule
    % \rowcolor{gray!30}
    & \multicolumn{5}{c}{\textbf{\texttt{Reuter News}}} & \multicolumn{5}{c}{\textbf{\texttt{Yelp Review}}} \\
    % \cmidrule(lr){2-6} \cmidrule(lr){7-11}
    & \datasetpaper & \textit{HumanEval} & \datasetessay & \datasetnews & \datasetreview & \datasetpaper & \textit{HumanEval} & \datasetessay & \datasetnews & \datasetreview \\
    \midrule
    Entropy & 36.08$\pm$0.845 & 39.51$\pm$0.573 & 23.58$\pm$0.794 & 46.70$\pm$0.440 & 38.95$\pm$0.125 & 
    36.06$\pm$0.607 & 27.85$\pm$1.194 & 37.80$\pm$1.538 & 37.13$\pm$0.789 & 48.10$\pm$0.949 \\
    \rowcolor{gray!10} Rank 
    & 42.79$\pm$1.454 & 31.06$\pm$0.644 & 39.18$\pm$1.365 & 49.40$\pm$1.573 & \underline{47.80$\pm$0.760} & 
    32.88$\pm$0.413 & 49.87$\pm$0.654 & 40.56$\pm$0.743 & 35.63$\pm$1.182 & 40.02$\pm$0.119 \\
    LogRank & 49.07$\pm$1.127 & 45.97$\pm$0.198 & 48.33$\pm$1.103 & 56.15$\pm$0.327 & 46.36$\pm$0.709 & 29.62$\pm$0.983 & 44.50$\pm$0.352 & 32.01$\pm$0.915 & 34.22$\pm$0.387 & 53.25$\pm$0.726 \\
    \rowcolor{gray!10} RoBERTa-base & 50.27$\pm$1.160 & 41.50$\pm$0.582 & 36.16$\pm$0.628 & 54.59$\pm$1.094 & 31.29$\pm$1.069 & 43.76$\pm$0.539 & 50.66$\pm$1.466 & 45.30$\pm$1.423 & 45.64$\pm$1.306 & 62.62$\pm$1.416 \\
    RoBERTa-large & 57.32$\pm$1.356 & \textbf{48.03$\pm$1.752} & 45.71$\pm$1.720 & 63.25$\pm$1.349 & 41.63$\pm$1.218 & \textbf{48.13$\pm$0.720} & 51.18$\pm$1.206 & 32.36$\pm$1.023 & 59.31$\pm$1.662 & 70.23$\pm$1.230 \\
    \rowcolor{gray!10} DetectGPT & 52.57$\pm$0.765 & 32.19$\pm$0.846 & 56.44$\pm$1.363 & 86.70$\pm$0.838 & 37.09$\pm$0.728 & 41.63$\pm$0.745 & 37.51$\pm$0.545 & 59.32$\pm$1.150 & 61.08$\pm$1.344 & 68.10$\pm$1.223 \\
    Ghostbuster & 58.71$\pm$1.388 & 39.02$\pm$0.411 & 60.61$\pm$0.596 & 82.56$\pm$0.386 & 41.30$\pm$1.749 & 40.12$\pm$0.313 & 32.10$\pm$0.267 & 57.42$\pm$1.451 & 44.28$\pm$0.651 & 66.34$\pm$0.268 \\
    Fast-DetectGPT & 62.75$\pm$1.107 & 45.73$\pm$1.216 & 68.14$\pm$0.821 & 81.73$\pm$0.593 & 45.18$\pm$0.818 & 46.60$\pm$0.988 & 50.22$\pm$1.361 & 66.80$\pm$0.863 & 62.08$\pm$0.390 & 68.03$\pm$1.103 \\
    \rowcolor{gray!10} RAIDAR & 60.37$\pm$0.441 & 33.07$\pm$1.711 & 59.74$\pm$0.300 & \textbf{89.67$\pm$0.372} & 32.23$\pm$0.759 & 43.09$\pm$1.738 & 42.31$\pm$1.795 & \underline{67.79$\pm$1.651} & 47.36$\pm$0.397 & 71.88$\pm$0.227 \\
    R-Detect & \underline{64.15$\pm$0.818} & 45.42$\pm$1.198 & \underline{71.09$\pm$1.031} & 77.42$\pm$1.067 & \underline{46.31$\pm$0.992} & 44.54$\pm$0.958 & \underline{52.01$\pm$0.810} & 67.51$\pm$0.605 & \underline{63.15$\pm$0.967} & \underline{72.40$\pm$0.585} \\
    \midrule
    \rowcolor{creamyellow!50} 
    \textbf{\ourmethod (Ours)} & \textbf{69.81$\pm$0.336} & \underline{46.35$\pm$0.914} & \textbf{82.61$\pm$0.847} & \underline{88.21$\pm$0.845} & \textbf{49.17$\pm$0.558} & \underline{46.67$\pm$0.672} & \textbf{53.43$\pm$0.634} & \textbf{69.03$\pm$0.316} & \textbf{69.88$\pm$0.538} & \textbf{80.83$\pm$0.305} \\ % OOD-Avg
    \midrule % \midrule   
    % \rowcolor{gray!30} 
    & \multicolumn{5}{c}{\textbf{\texttt{Student Essay}}} & \multicolumn{5}{c}{\textbf{\texttt{OOD-Avg}}}  \\
    % \cmidrule(lr){2-6} \cmidrule(lr){7-11}
    & \datasetpaper & \textit{HumanEval} & \datasetessay & \datasetnews & \datasetreview & \datasetpaper & \textit{HumanEval} & \datasetessay & \datasetnews & \datasetreview \\
    \midrule
    Entropy & 48.15$\pm$0.565 & 26.93$\pm$1.397 & 36.68$\pm$0.753 & 41.05$\pm$1.102 & 31.16$\pm$0.766 
    & 40.63$\pm$0.751 & 36.13$\pm$0.995 & 34.05$\pm$0.863 & 41.84$\pm$0.658 & 40.05$\pm$0.827 \\
    \rowcolor{gray!10} Rank 
    & 51.23$\pm$0.875 & 26.18$\pm$0.540 & 51.75$\pm$0.310 & 43.22$\pm$0.282 & 45.61$\pm$0.472 
    & 43.32$\pm$0.832 & 37.32$\pm$0.512 & 41.44$\pm$0.755 & 39.07$\pm$0.981 & 45.66$\pm$0.563 \\
    LogRank & 48.86$\pm$0.527 & 31.23$\pm$1.202 & 59.10$\pm$1.080 & 46.77$\pm$1.800 & 41.67$\pm$0.319 
    & 46.90$\pm$0.811 & 40.58$\pm$0.623 & 43.36$\pm$0.976 & 43.73$\pm$0.915 & 45.66$\pm$0.736 \\
    \rowcolor{gray!10} RoBERTa-base & 58.55$\pm$0.555 & 32.32$\pm$1.379 & 45.14$\pm$0.663 & 34.26$\pm$0.380 & \underline{50.52$\pm$0.842}
    & 49.87$\pm$0.683 & 41.99$\pm$1.066 & 39.52$\pm$0.765 & 44.46$\pm$0.811 & 49.82$\pm$1.047 \\
    RoBERTa-large & 53.10$\pm$1.145 & 28.66$\pm$1.066 & 61.93$\pm$0.583 & 37.82$\pm$1.806 & 47.04$\pm$0.736 
    & 52.22$\pm$1.001 & 45.49$\pm$1.188 & 47.75$\pm$1.047 & 51.05$\pm$1.334 & 54.56$\pm$1.030 \\
    \rowcolor{gray!10} DetectGPT & 50.41$\pm$0.976 & 31.40$\pm$0.761 & 70.56$\pm$1.597 & 64.34$\pm$1.043 & 41.54$\pm$0.590 
    & 53.87$\pm$0.865 & 38.19$\pm$0.779 & 61.97$\pm$1.099 & 64.64$\pm$0.998 & 50.60$\pm$0.792 \\
    Ghostbuster & 63.97$\pm$0.404 & \underline{35.63$\pm$1.330} & 71.47$\pm$1.503 & 70.18$\pm$0.874 & 38.73$\pm$0.357 
    & 55.39$\pm$0.701 & 39.84$\pm$0.642 & 60.34$\pm$1.095 & 61.30$\pm$0.734 & 52.51$\pm$0.803 \\
    Fast-DetectGPT & 55.29$\pm$1.107 & 34.16$\pm$0.983 & 73.52$\pm$1.210 & 66.59$\pm$0.792 & 50.06$\pm$0.385 & 59.24$\pm$1.106 & 45.81$\pm$1.152 & 66.98$\pm$1.125 & 66.79$\pm$0.709 & 58.09$\pm$0.885 \\
    \rowcolor{gray!10} RAIDAR & 64.75$\pm$1.432 & 34.69$\pm$0.660 & 72.02$\pm$0.994 & \underline{74.88$\pm$1.546} & 36.97$\pm$0.323 
    & 58.57$\pm$0.953 & 44.43$\pm$1.010 & 60.13$\pm$0.907 & \underline{68.99$\pm$0.797} & 56.29$\pm$0.598 \\
    R-Detect & \underline{66.89$\pm$0.271} & 33.06$\pm$0.939 & \underline{76.44$\pm$1.095} & 72.33$\pm$0.426 & 49.05$\pm$0.714 & \underline{60.73$\pm$0.690} & \underline{49.50$\pm$0.981} & \underline{67.78$\pm$0.946} & 66.61$\pm$0.806 & \underline{58.86$\pm$0.708} \\
    \midrule
    \rowcolor{creamyellow!50} 
    \textbf{\ourmethod (Ours)} & 
    \textbf{71.31$\pm$0.892} & \textbf{38.52$\pm$0.334} & \textbf{89.27$\pm$1.532} & \textbf{76.07$\pm$0.930} & 
    \textbf{53.15$\pm$0.570} & \textbf{65.15$\pm$0.585} % $_{\textcolor{black}{\uparrow6.5}}$
    & \textbf{53.35$\pm$0.555} % $_{\textcolor{black}{\uparrow7.8}}$
    & \textbf{75.77$\pm$0.613} % $_{\textcolor{black}{\uparrow13.8}}$
    & \textbf{72.68$\pm$0.648} % $_{\textcolor{black}{\uparrow3.6}}$
    & \textbf{64.14$\pm$0.438} \\ % $_{\textcolor{black}{\uparrow7.8}}$
    \bottomrule
    \end{tabular}
    }
\end{table*}
While the in-domain results demonstrate strong discriminative ability, practical deployment also requires detectors to transfer across domains with different writing styles and content distributions. We therefore next evaluate cross-domain generalization.
We perform experiments using ID-OOD splits, where detectors are trained on a source domain and evaluated on unseen target domains.
As summarized in~\autoref{figure:OOD}, the results show that \ourmethod consistently achieves the highest OOD-Avg F1 score across all source domains, highlighting its strong ability to generalize to new, unseen domains.
When trained on formal domains such as \datasetpaper and \datasetessay, our method generalizes well to informal targets like \datasetnews and \datasetreview, demonstrating its versatility across different writing styles. 
Similarly, when trained on casual domains like \datasetreview, it maintains stable performance even when tested on more formal domains like \datasetpaper and \datasetessay.
Although cross-domain detection remains a challenging task for all methods, the proposed stability-based approach consistently outperforms other methods, showing stronger transferability under domain shifts.
Notably, we observe that machine-generated content consistently exhibits higher mean values than human-written content across each component of the consistency profile. 
This observation aligns with the consistency dominance assumption stated in~\autoref{theorem}, which posits a coordinate-wise separation of class-conditional expectations in the consistency profile space.

\subsection{Robustness Analysis}
Beyond domain transfer, a robust detector should also remain reliable under post-editing and mixed-authorship settings. 
We therefore further evaluate \ourmethod under adversarial perturbations and hybrid human-LLM composition.
\paragraph{Robustness to Adversarial Manipulation.}
\begin{figure}[!t]
    \centering
    \includegraphics[width=\linewidth]{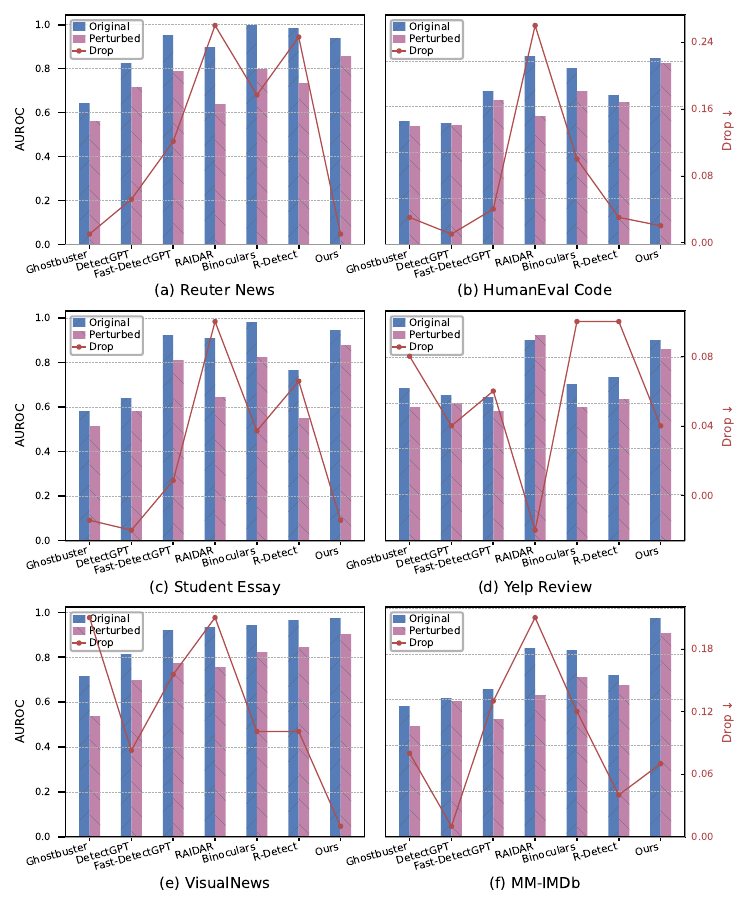}
    \caption{Evaluation of detection performance degradation under adversarial perturbations: Comparative analysis of original AUROC, AUROC after perturbation, and relative drop rate across \datasetnews, \datasetcode, \datasetessay,  \datasetreview, \datasetvisualnews, and \datasetmmimdb datasets, highlighting robustness under word-level attacks.
    }
    \label{figure:perturbation}
\end{figure}
To assess robustness against meaning-preserving edits, we adopt TextAttack-style perturbations and introduce character swaps/insertions, synonym-level word substitutions, and sentence-level paraphrases, with a maximum modification rate of 20\% tokens per sample. 
All detectors are trained on clean data and tested directly on perturbed sets without adaptation, reflecting real-world deployment where edited or partially rewritten text is common. 
As illustrated in~\autoref{figure:perturbation}, we report original AUROC, post-perturbation AUROC, and the relative performance drop. 
Across domains, \ourmethod consistently incurs substantially smaller degradation than likelihood-based or probability–dependent baselines. 
Notably, detectors such as DetectGPT and Ghostbuster exhibit pronounced drops, especially in \datasetreview and \datasetcode, where perturbations disrupt probability curvature or token statistics, whereas \ourmethod maintains high accuracy with only minor fluctuations. 
This pattern holds consistently across both conventional text domains and multimedia content domains (\datasetvisualnews and \datasetmmimdb), demonstrating that stylistic consistency remains a reliable signal regardless of content modality.
% In several cases (e.g., \datasetreview), our performance remains nearly unchanged after perturbation, implying that consistency profiling effectively filters out superficial textual edits and preserves authorship cues beyond raw token-form information. 
% These findings suggest that our stability-oriented representation captures deeper stylistic invariants rather than surface features., and hence remains reliable when content undergoes light rewriting or noisy edits—an ability crucial for deployment scenarios involving user post-editing, LLM-based paraphrasing, or content laundering. 
This addresses our goal of robust detection (\textit{RQ1}), confirming the resilience of stylistic consistency under adversarial manipulation.

\paragraph{Robustness to Hybrid Human-LLM Composition.}
\begin{figure}[!t]
    \centering
    \includegraphics[width=\linewidth]{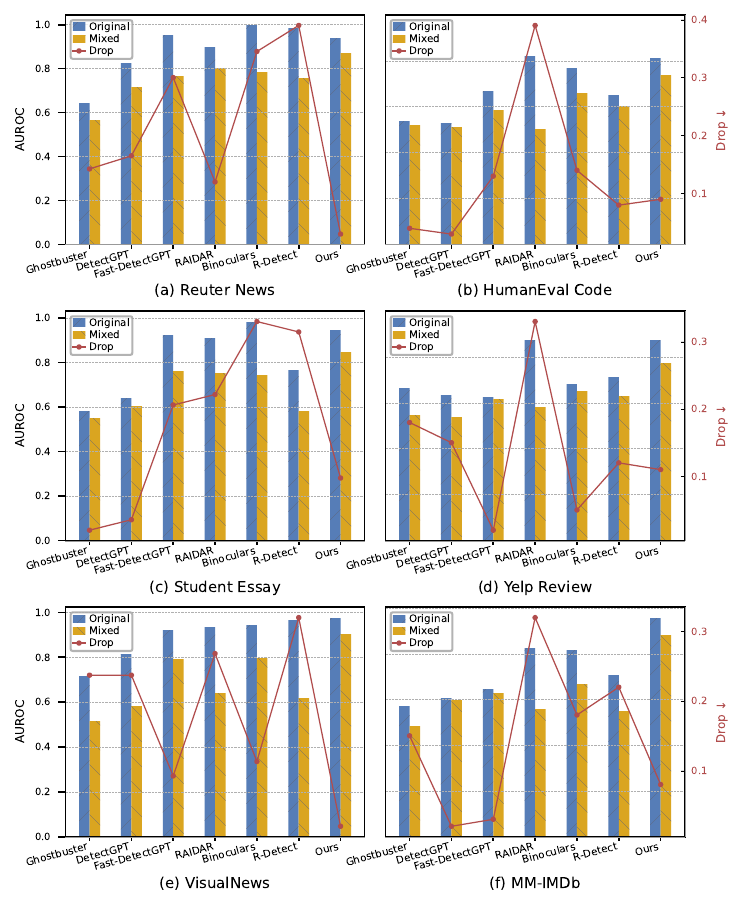}
    \caption{Evaluation of detection performance degradation under adversarial mixed text: Comparative analysis of original AUROC, AUROC after mixing, and relative performance drop across \datasetnews, \datasetcode, \datasetessay,  \datasetreview, \datasetvisualnews and \datasetmmimdb datasets, characterizing stability under hybrid human–LLM composition.}
    \label{figure:mixed}
\end{figure}
Beyond local perturbations, we further evaluate robustness under global content mixing, where human-written and LLM-generated segments are interleaved within a single document. 
Following standard protocol, we construct hybrid samples by concatenating segments at a 4:1 ratio and assign labels based on dominant authorship. 
This setup reflects situations where users revise LLM outputs or insert generated paragraphs into human writing. 
The results in~\autoref{figure:mixed} show that \ourmethod continues to outperform all baselines on mixed inputs and maintains the smallest AUROC drop across domains. 
While detectors relying on perplexity or representation distance degrade significantly under blending, often losing the authorship signal once machine spans are surrounded by human context, our stylistic consistency profile remains discriminative even without span-level supervision. 
For example, on \datasetreview, most baselines experience severe degradation, while \ourmethod drops only modestly, indicating strong resilience to human–AI hybridization. 
%
% These observations demonstrate that consistency across paraphrased variants remains a stable indicator of machine authorship even when machine-generated text is partially masked or diluted by human content. 
% Combined with perturbation experiments, this confirms that \ourmethod supports detection not only under local edits but also under mixed scenarios, addressing \textit{RQ2} and highlighting the practicality in real-world moderation, academic integrity checking, and hybrid writing workflows.

Combined with perturbation experiments, this confirms that \ourmethod supports detection not only under local edits but also under mixed scenarios, addressing 
\textit{RQ2} and highlighting its practicality across diverse real-world multimedia content moderation scenarios.

\subsection{Interpretability Analysis}
Besides the robustness of \ourmethod, we next analyze whether the learned stability profile also yields an interpretable feature-space structure for more transparency.
\paragraph{Visualizing Explainability through UMAP}
\begin{figure}[!t]
    \centering
    \footnotesize
    \setlength{\tabcolsep}{2.0pt}
    \includegraphics[width=\linewidth]{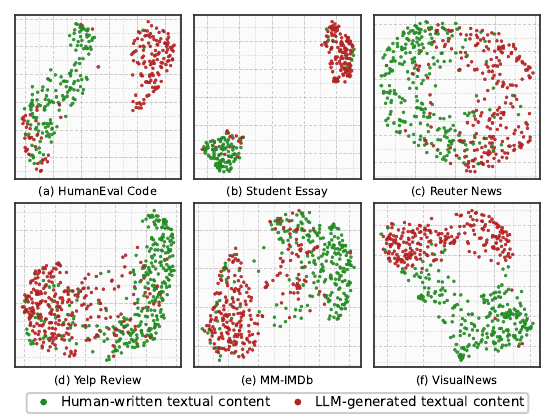}
    \caption{Explainable in-domain detection visualization: UMAP projections of stability signatures separating human-written and LLM-generated textual content across six domains (\datasetreview, \datasetnews, \datasetcode, \datasetessay, \datasetvisualnews, and \datasetmmimdb).}
    \label{figure:explainability-UMAP}
\end{figure}
To highlight the explainability of our proposed method, \ourmethod, we utilize UMAP~\cite{mcinnes2020umapuniformmanifoldapproximation} to visualize the distribution of feature vectors extracted from various datasets. 
As shown in~\autoref{figure:explainability-UMAP}, UMAP provides a two-dimensional projection that facilitates the analysis of high-dimensional data, offering a feature-space sanity check on how our stability profile organizes texts from different sources rather than serving as a decision tool.
As shown in~\autoref{figure:explainability-UMAP}, the red points represent machine-generated content and the green points correspond to human-written content. 
Across all domains, the two groups are clearly separated, with limited overlap in the projected space. Notably, \datasetessay exhibits a particularly clean separation, indicating that \ourmethod captures style-based differences effectively even in relatively complex content domains. 
This pattern also extends to multimedia content domains: \datasetmmimdb and \datasetvisualnews both display clear cluster boundaries, suggesting that the learned consistency signatures remain informative across heterogeneous settings and supporting the generalizability of \ourmethod in real-world content moderation scenarios.

\paragraph{Quantitative Evaluation}
\begin{figure}[!t]
    \centering
    \footnotesize
    \includegraphics[width=\linewidth]{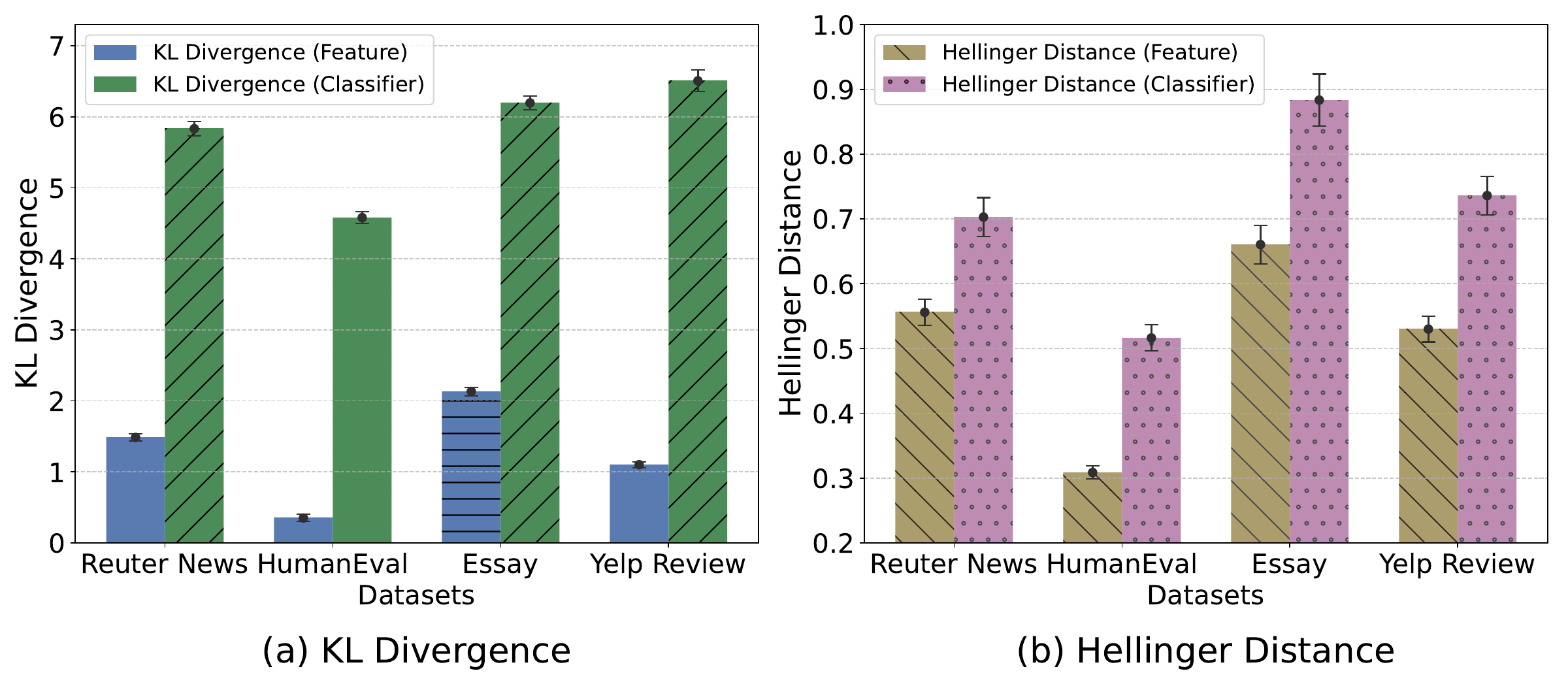}
    \caption{Explainability experiments using KL Divergence and Hellinger Distance for feature and classifier-based methods across multiple datasets}
    \label{figure:explainability-metric}
\end{figure}
While the UMAP visualization in the previous section provided a qualitative view of the feature space separation, we quantify the model's ability to distinguish between human and machine content using two distribution divergence metrics: KL Divergence and Hellinger Distance.
These metrics measure the divergence between the distributions of human-written and LLM-generated textual content, providing a deeper understanding of how the model distinguishes between the two. 
The results are presented in \autoref{figure:explainability-metric}, where two methods are used: the first derives distributions from the feature $v(x)$ extracted  by \ourmethod, and the second uses classifier prediction scores $\sigma(\hat y)$.
The results consistently show that the classifier-based method (using $\sigma(\hat y)$) outperforms the feature-based method (using $v(x)$) in both KL Divergence and Hellinger Distance across all datasets. 
In particular, the KL Divergence is significantly higher for the distribution $\sigma (\hat y)$, especially in domains like \datasetreview, suggesting that the classifier captures more distinct stylistic differences. 
Similarly, higher Hellinger Distance values for the distribution $\sigma (\hat y)$ indicate clearer separability between human-written and LLM-generated textual content.
This gap between $v(x)$ and $\sigma(\hat{y})$ is expected: the stability profile provides a compact, interpretable representation, while the classifier $f$ further amplifies the separation by learning non-linear decision boundaries. Together, they support both transparent feature-space inspection and strong end-to-end detection performance.
These findings demonstrate that \ourmethod not only achieves high detection accuracy but also enhances transparency and explainability.

\subsection{Ablation Study}
Finally, to understand how much the framework depends on the specific continuous representation module, we conduct an ablation study over different semantic encoders.
Specifically, we conduct a component ablation study by replacing the continuous representation module $\xi$ with different feature extractors ranging from lightweight statistical vectors to deep contextual encoders. 
As shown in \autoref{table:pluggability}, \ourmethod~maintains consistently high performance across encoder choices. 
This evaluates the plug-and-play compatibility of \ourmethod and verifies that the stylistic consistency profiling remains effective when using weaker or stronger semantic embeddings.

Interestingly, even when using TF-IDF, a purely statistical and non-contextual representation without deep semantics, our method still delivers competitive results.
Replacing TF-IDF with pretrained distributed embeddings (Word2Vec/GloVe), \ourmethod obtains improved results through richer lexical features, suggesting that capturing global lexical semantics benefits profile construction. 
Among these semantic encoders, Contextual encoders (BERT, SBERT) are particularly effective, with SBERT providing the best overall average AUROC when plugged into \ourmethod. 
These results confirm that our mechanism is encoder-agnostic, providing flexibility for various deployment scenarios.

\begin{table}[!t]
  \centering
  \footnotesize
  \setlength{\tabcolsep}{4pt}
  \caption{AUROC results of replacing the semantic encoder with different representations across domains. 
  \ourmethod remains effective across feature extractors, confirming the plug-and-play capability.}
  \label{table:pluggability}
  \resizebox{\linewidth}{!}{
  \begin{tabular}{lcccc}
    \toprule
    \textbf{Encoder} & \textbf{Reuter News} & \textbf{HumanEval} & \textbf{Essay} & \textbf{Yelp Review} \\
    \midrule
    TF-IDF              & 0.8683 & 0.7709 & 0.9285 & 0.8114 \\
    Word2Vec/GloVe & 0.8864 & 0.8087 & 0.9369 & 0.8309 \\
    BERT                & 0.9300 & 0.7812 & 0.9478 & 0.8566 \\
    SBERT (Default) & \textbf{0.9356} & \textbf{0.8108} &  \textbf{0.9455} & \textbf{0.8718} \\
    \bottomrule
  \end{tabular}
  }
\end{table}

\section{Conclusion}
In this work, we proposed \ourmethod, a lightweight framework for detecting LLM-generated textual content through stylistic consistency profiling.
By combining discrete stylistic signals with continuous semantic consistency, \ourmethod provides a compact representation that remains effective under paraphrase-based variation and multimodal-guided rewriting.
Experiments across conventional and multimedia-associated domains show that \ourmethod achieves strong detection performance and robust behavior under adversarial perturbations and hybrid human-LLM composition.
It also consistently outperforms existing methods in cross-domain evaluations, demonstrating strong generalization under distribution shift.
Further analyses show that \ourmethod remains effective across different semantic encoders, supporting flexible deployment.

\setcitestyle{numbers,sort&compress}
\bibliographystyle{unsrtnat}
\bibliography{paper}

\appendix

\end{document}